\documentclass{llncs}
\usepackage{llncsdoc}

\usepackage{amsmath}
\usepackage{amsfonts}
\usepackage[dvipsnames]{xcolor}
\usepackage{graphicx}
\usepackage{marvosym}
\usepackage{hyperref}

\title{A Thorough Formalization of Conceptual Spaces\thanks{The final publication is available at Springer via \url{http://dx.doi.org/10.1007/978-3-319-67190-1_5}}}
\author{Lucas Bechberger\thanks{Corresponding author, ORCID: 0000-0002-1962-1777} \and Kai-Uwe K\"uhnberger}
\institute{Institute of Cognitive Science, Osnabr\"uck University, Osnabr\"uck, Germany \email{lucas.bechberger@uni-osnabrueck.de}, \email{kai-uwe.kuehnberger@uni-osnabrueck.de}}

\begin{document}
\maketitle

\begin{abstract}
The highly influential framework of conceptual spaces provides a geometric way of representing knowledge. Instances are represented by points in a high-dimensional space and concepts are represented by convex regions in this space. After pointing out a problem with the convexity requirement, we propose a formalization of conceptual spaces based on fuzzy star-shaped sets. Our formalization uses a parametric definition of concepts and extends the original framework by adding means to represent correlations between different domains in a geometric way. Moreover, we define computationally efficient operations on concepts (intersection, union, and projection onto a subspace) and show that these operations can support both learning and reasoning processes.
\end{abstract}
\begin{keywords} Conceptual Spaces \textperiodcentered\; Star-Shaped Sets \textperiodcentered\; Fuzzy Sets
\end{keywords}
\section{Introduction}
\label{Intro}

One common criticism of symbolic AI approaches is that the symbols they operate on do not contain any meaning: For the system, they are just arbitrary tokens that can be manipulated in some way. This lack of inherent meaning in abstract symbols is called the “symbol grounding problem” \cite{Harnad1990}. One approach towards solving this problem is to devise a grounding mechanism that connects  abstract symbols to the real world, i.e., to perception and action.\\

The framework of conceptual spaces \cite{Gardenfors2000,Gardenfors2014} attempts to bridge this gap between symbolic and subsymbolic AI by proposing an intermediate conceptual layer based on geometric representations.
A conceptual space is a high-dimensional space spanned by a number of quality dimensions that are based on perception and/or subsymbolic processing. Convex regions in this space correspond to concepts. Abstract symbols can thus be grounded in reality by linking them to regions in a conceptual space whose dimensions are based on perception.

The framework of conceptual spaces has been highly influential in the last 15 years within cognitive science and cognitive linguistics \cite{Douven2011,Fiorini2013,Warglien2012}. It has also sparked considerable research in various subfields of artificial intelligence, ranging from robotics and computer vision \cite{Chella2005,Chella2001,Chella2003} over the semantic web and ontology integration \cite{Adams2009a,Dietze2008} to plausible reasoning \cite{Derrac2015,Schockaert2011}.\\

One important question is however left unaddressed by these research efforts: How can an (artificial) agent learn about meaningful regions in a conceptual space purely from unlabeled perceptual data?

Our approach for solving this concept formation problem is to devise an incremental clustering algorithm that groups a stream of unlabeled observations (represented as points in a conceptual space) into meaningful regions.

In this paper, we point out that G\"ardenfors' convexity requirement prevents a geometric representation of correlations. We resolve this problem by using star-shaped instead of convex sets. Our mathematical formalization defines concepts in a parametric way that is easily implementable. We furthermore define computationally efficient operations on these concepts, which can support both machine learning and reasoning processes. This paper therefore lays the foundation for our work on concept formation.

The remainder of this paper is structured as follows:
Section \ref{CS} introduces the general framework of conceptual spaces and points out a problem with the notion of convexity. Section \ref{FSSSS} describes our formalization of concepts as fuzzy star-shaped sets. In Section \ref{Operations}, we define operations on these sets and in Section \ref{Applications} we show that they can support both machine learning and reasoning processes. Section \ref{RelatedWork} summarizes related work and Section \ref{Conclusion} concludes the paper. Proofs of our propositions are provided in an appendix avaliable online at \url{http://lucas-bechberger.de/appendix-ki-2017/}.

\section{Conceptual Spaces}
\label{CS}

\subsection{Definition of Conceptual Spaces}
\label{CS:Definition}

This section presents the cognitive framework of conceptual spaces as described in \cite{Gardenfors2000} and introduces our formalization of dimensions, domains, and distances.

A conceptual space is a high-dimensional space spanned by a set $D$ of so-called ``quality dimensions''. Each of these dimensions $d \in D$ represents a way in which two stimuli can be judged to be similar or different. Examples for quality dimensions include temperature, weight, time, pitch, and hue. We denote the distance between two points $x$ and $y$ with respect to a dimension $d$ as $|x_d - y_d|$.

A domain $\delta \subseteq D$ is a set of dimensions that inherently belong together. Different perceptual modalities (like color, shape, or taste) are represented by different domains. The color domain for instance consists of the three dimensions hue, saturation, and brightness.

G\"{a}rdenfors argues based on psychological evidence \cite{Attneave1950,Shepard1964} that distance within a domain $\delta$ should be measured by the weighted Euclidean metric: 
$$d_E^{\delta}(x,y, W_{\delta}) = \sqrt{\sum_{d \in \delta} w_{d} \cdot | x_{d} - y_{d} |^2}$$
The parameter $W_{\delta}$ contains positive weights $w_{d}$ for all dimensions $d \in \delta$ representing their relative importance. We assume that $\textstyle\sum_{d \in \delta} w_{d} = 1$.\\


The overall conceptual space $CS$ is defined as the product space of all dimensions. Again, based on psychological evidence \cite{Attneave1950,Shepard1964}, G\"{a}rdenfors argues that distance within the overall conceptual space should be measured by the weighted Manhattan metric $d_M$ of the intra-domain distances. Let $\Delta$ be the set of all domains in $CS$. We define the distance within a conceptual space as follows:
$$
d_C^{\Delta}(x,y,W) = \sum_{\delta \in \Delta} w_{\delta} \cdot d_E^{\delta}(x,y,W_{\delta})
= \sum_{\delta \in \Delta}w_{\delta} \cdot \sqrt{\sum_{d \in \delta} w_{d} \cdot |x_{d} - y_{d}|^2}
$$
The parameter $W\hspace{-0.2cm} = \hspace{-0.2cm}\langle W_{\Delta},\{W_{\delta}\}_{\delta \in \Delta}\rangle$ contains $W_{\Delta}$, the set of positive domain weights $w_{\delta}$. We require that $\textstyle\sum_{\delta \in \Delta} w_{\delta} = |\Delta|$. Moreover, $W$ contains for each domain $\delta \in \Delta$ a set $W_{\delta}$ of dimension weights as defined above. The weights in $W$ are not globally constant, but depend on the current context. One can easily show that $d_C^{\Delta}(x,y,W)$ with a given $W$ is a metric.\\

The similarity of two points in a conceptual space is inversely related to their distance. G\"{a}rdenfors expresses this as follows :
$$Sim(x,y) = e^{-c \cdot d(x,y)}\quad \text{with a constant}\; c >0 \; \text{and a given metric}\; d$$

Betweenness is a logical predicate $B(x,y,z)$ that is true if and only if $y$ is considered to be between $x$ and $z$. It can be defined based on a given metric $d$: 
$$B_d(x,y,z) :\iff d(x,y) + d(y,z) = d(x,z)$$

The betweenness relation based on $d_E$ results in the line segment connecting the points $x$ and $z$, whereas the betweenness relation based on $d_M$ results in an axis-parallel cuboid between the points $x$ and $z$.
We can define convexity and star-shapedness based on the notion of betweenness:

\begin{definition}
\label{def:Convexity}
(Convexity)\\
A set $C \subseteq CS$ is \emph{convex} under a metric $d \;:\iff$

\hspace{1cm}$\forall {x \in C, z \in C, y \in CS}: \left(B_d(x,y,z) \rightarrow y \in C\right)$
\end{definition}

\begin{definition}
\label{def:StarShapedSet}
(Star-shapedness)\\
A set $S \subseteq CS$ is \emph{star-shaped} under a metric $d$ with respect to a set $P \subseteq S \;:\iff$ 

\hspace{1cm}$\forall {p \in P, z \in S, y \in CS}: \left(B_d(p,y,z) \rightarrow y \in S\right)$
\end{definition}

G\"{a}rdenfors distinguishes properties like ``red'', ``round'', and ``sweet'' from full-fleshed concepts like ``apple'' or ``dog'' by observing that properties can be defined on individual domains (e.g., color, shape, taste), whereas full-fleshed concepts involve multiple domains.

\begin{definition}
\label{def:CriterionP}
(Property)\\
A \emph{natural property} is a convex region of a domain in a conceptual space.
\end{definition}

Full-fleshed concepts can be expressed as a combination of properties from different domains. These domains might have a different importance for the concept which is reflected by so-called ``salience weights''. Another important aspect of concepts are the correlations between the different domains \cite{Medin1988}, which are important for both learning \cite{Billman1996} and reasoning \cite[Ch 8]{Murphy2002}.

\begin{definition}
\label{def:CriterionC}
(Concept)\\
A \emph{natural concept} is represented as a set of convex regions in a number of domains together with an assignment of salience weights to the domains and information about how the regions in different domains are correlated.
\end{definition}

\subsection{An Argument Against Convexity}
\label{CS:ArgumentAgainstConvexity}

\begin{figure}[tp]
\centering
\includegraphics[width=0.75\columnwidth]{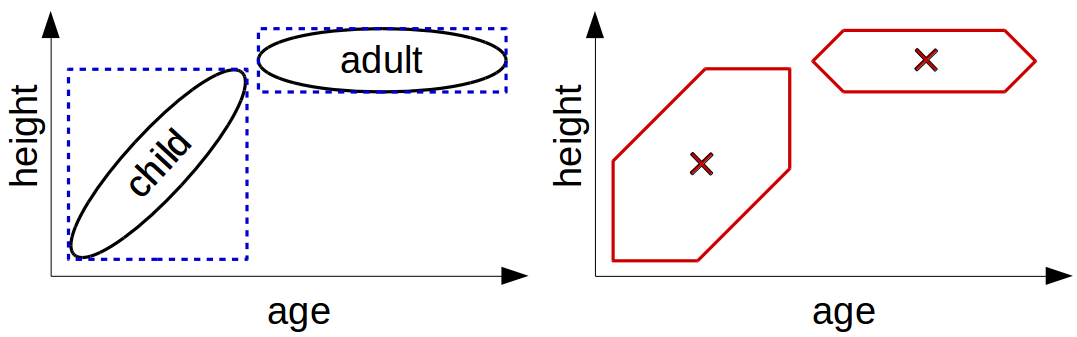}
\caption{Left: Intuitive way to define regions for the concepts of ``adult'' and ``child'' (solid) as well as representation by using convex sets (dashed). Right: Representation by using star-shaped sets with central points marked by crosses.}
\label{fig:ConvexityProblem}
\end{figure}

G\"{a}rdenfors \cite{Gardenfors2000} does not propose any concrete way for representing correlations between domains. As the main idea of the conceptual spaces framework is to find a geometric representation of conceptual structures, we think that a geometric representation of these correlations is desirable.

Consider the left part of Figure \ref{fig:ConvexityProblem}. In this example, we consider two domains, age and height, in order to define the concepts of child and adult. We would expect a strong correlation between age and height for children, but no such correlation for adults. This is represented by the two solid ellipses.

Domains are combined by using the Manhattan metric and convex sets under the Manhattan metric are axis-parallel cuboids. Thus, a convex representation of the two concepts results in the dashed rectangles. This means that we cannot geometrically represent correlations between domains if we assume that concepts are convex and that the Manhattan metric is used. We think that our example is not a pathological one and that similar problems will occur quite frequently when encoding concepts. From a different perspective, also Hern{\'{a}}ndez-Conde has recently argued against the convexity constraint in conceptual spaces \cite{Hernandez-Conde2016}.\\

If we only require star-shapedness instead of convexity, we can represent the correlation of age and height for children in a geometric way. This is shown in the right part of Figure \ref{fig:ConvexityProblem}: Both sketched sets are star-shaped under the Manhattan metric with respect to a central point. Although the star-shaped sets do not exactly correspond to our intuitive sketch in the left part of Figure \ref{fig:ConvexityProblem}, they definitely are an improvement over the convex representation.\footnote{The weaker requirement of star-shapedness allows us to ``cut out'' some corners from the rectangle. This enables us to geometrically represent correlations.}

Star-shaped sets cannot contain any ``holes''. They furthermore have a well defined central region $P$ that can be interpreted as a prototype. Thus, the connection that was established between the prototype theory of concepts and the framework of conceptual spaces \cite{Gardenfors2000} is preserved. Replacing convexity with star-shapedness is therefore only a minimal departure from the original framework.\\

The problem illustrated in Figure \ref{fig:ConvexityProblem} could also be resolved by using the Euclidean metric instead of the Manhattan metric for combining domains. We think however that this would be a major modification of the original framework. For instance, if the use of the Manhattan metric is abolished, the usage of domains to structure the conceptual space loses its main effect of influencing the overall distance metric. 
Moreover, psychological evidence \cite{Attneave1950,Shepard1964,Shepard1987} indicates that human similarity ratings are reflected better by the Manhattan metric than by the Euclidean metric if different domains are involved (e.g., stimuli differing in size and brightness). As a psychologically plausible representation of similarity is one of the core principles of the conceptual spaces framework, these findings should be taken into account.
Furthermore, in high-dimensional feature spaces the Manhattan metric provides a better relative contrast between close and distant points than the Euclidean metric \cite{Aggarwal2001}. If we expect the number of domains to be large, this also supports the usage of the Manhattan metric from an implementational point of view.

Based on these arguments, we think that weakening the convexity assumption is a better option than abolishing the use of the Manhattan metric.

\section{A Parametric Definition of Concepts}
\label{FSSSS}
\subsection{Preliminaries}
\label{FSSSS:Preliminaries}

Our formalization is based on the following insight: 

\begin{lemma}
\label{lemma:UnionOfConvex}
Let $C_1, ..., C_m$ be convex sets in $CS$ under some metric $d$ and let $P := \textstyle\bigcap_{i=1}^{m} C_i$. If $P \neq \emptyset$, then $S := \textstyle\bigcup_{i=1}^{m} C_i$ is star-shaped under $d$ w.r.t. $P$.
\end{lemma}
\begin{proof}
Obvious (see also \cite{Smith1968}).
\end{proof}

We will use axis-parallel cuboids as building blocks for our star-shaped sets. They are defined in the following way:

\begin{definition}
\label{def:Cuboid}
(Axis-parallel cuboid)\\
We describe an \emph{axis-parallel cuboid}\footnote{We will drop the modifier "axis-parallel" from now on.} $C$ as a triple $\langle\Delta_C, p^-, p^+\rangle$. $C$ is defined on the domains $\Delta_C \subseteq \Delta$, i.e. on the dimensions $D_C = \textstyle\bigcup_{\delta \in \Delta_C} \delta$. We call $p^-, p^+ $ the support points of $C$ and require:
$$
\forall {d \in D_C}: p^+_d,p^-_d \notin \{+\infty,-\infty\} \quad \land \quad
\forall {d \in D \setminus D_C}: p^-_d := -\infty \land p^+_d := +\infty
$$
Then, we define the cuboid $C$ in the following way:
$$C = \{x \in CS \;|\; \forall {d \in D}: p^-_d \leq x_d \leq p^+_d\}$$
\end{definition}

\begin{lemma}
\label{lemma:Cuboid}
A cuboid $C$ is convex under $d_C^{\Delta}$, given a fixed set of weights $W$.
\end{lemma}
\begin{proof}
It is easy to see that cuboids are convex with respect to $d_M$ and $d_E$. Based on this, one can show that they are also convex with respect to $d_C^{\Delta}$, which is a combination of $d_M$ and $d_E$.
\end{proof}

Our formalization will make use of fuzzy sets \cite{Zadeh1965}, which can be defined in our current context as follows:

\begin{definition}
(Fuzzy set)\\
A \emph{fuzzy set} $\widetilde{A}$ on $CS$ is defined by its membership function $\mu_{\widetilde{A}}: CS \rightarrow [0,1]$.
\end{definition}

Note that fuzzy sets contain crisp sets as a special case where $\mu_{\widetilde{A}}: CS \rightarrow \{0,1\}$. For each $x \in CS$, we interpret $\mu_{\widetilde{A}}(x)$ as degree of membership of $x$ in $\widetilde{A}$. 

\begin{definition}
(Alpha-cut)\\
Given a fuzzy set $\widetilde{A}$ on $CS$, its \emph{$\alpha$-cut} ${\widetilde{A}}^{\alpha}$ for $\alpha \in [0,1]$ is defined as follows:
$${\widetilde{A}}^{\alpha} = \{x \in CS\; |\; \mu_{\widetilde{A}}(x) \geq \alpha\}$$
\end{definition}

\begin{definition}
\label{def:FuzzyStarShaped}
(Fuzzy star-shapedness)\\
A fuzzy set $\widetilde{A}$ is called \emph{star-shaped} under a metric $d$ with respect to a (crisp) set $P$ if all of its $\alpha$-cuts ${\widetilde{A}}^{\alpha}$ are either empty or star-shaped under $d$ w.r.t. $P$.
\end{definition}

One can also generalize the ideas of subsethood, intersection, and union from crisp to fuzzy sets. We adopt the most widely used definitions:

\begin{definition}
\label{def:FuzzyOperations}
(Operations on fuzzy sets)\\
Let $\widetilde{A}, \widetilde{B}$ be two fuzzy sets defined on $CS$.
\begin{itemize}
	\item Subsethood: \hspace{0.20cm}$\widetilde{A} \subseteq \widetilde{B} :\iff (\forall {x \in CS}: \mu_{\widetilde{A}}(x) \leq \mu_{\widetilde{B}}(x))$
	\item Intersection: \hspace{0.10cm}$\forall x \in CS: \mu_{\widetilde{A} \cap \widetilde{B}}(x) := \min(\mu_{\widetilde{A}}(x),\mu_{\widetilde{B}}(x))$
	\item Union: \hspace{0.97cm}$\forall x \in CS: \mu_{\widetilde{A} \cup \widetilde{B}}(x) := \max(\mu_{\widetilde{A}}(x),\mu_{\widetilde{B}}(x))$
\end{itemize}
\end{definition}

\subsection{Fuzzy Simple Star-Shaped Sets}
\label{FSSSS:FSSSS}

\begin{figure}[tp]
\centering
\includegraphics[width = \columnwidth]{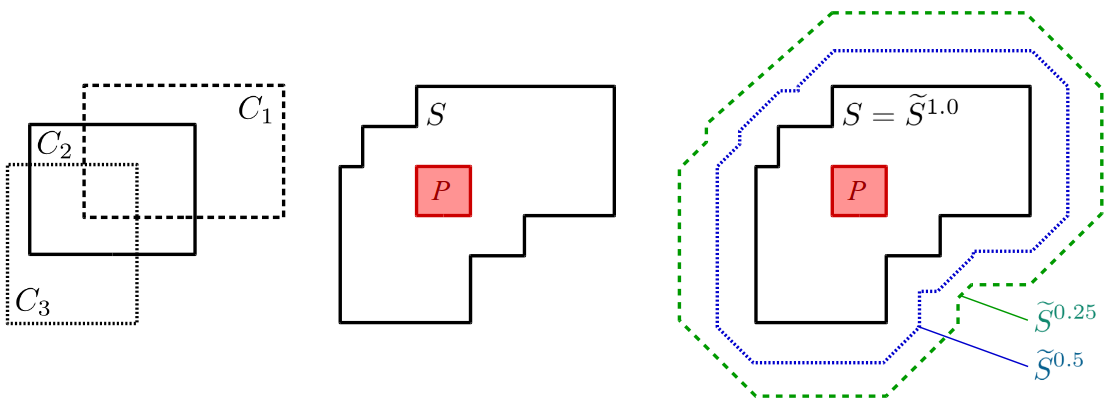} 
\caption{Left: Three cuboids $C_1, C_2, C_3$ with nonempty intersection. Middle: Resulting simple star-shaped set $S$ based on these cuboids. Right: Fuzzy simple star-shaped set $\tilde{S}$ based on $S$ with three $\alpha$-cuts for $\alpha \in \{1.0,0.5,0.25\}$.}
\label{fig:FSSSS}
\end{figure}

By combining Lemma \ref{lemma:UnionOfConvex} and Lemma \ref{lemma:Cuboid}, we see that any union of intersecting cuboids is star-shaped under $d_C^{\Delta}$. We use this insight to define simple star-shaped sets (illustrated in Figure \ref{fig:FSSSS}):

\begin{definition}
\label{def:SSSS}
(Simple star-shaped set)\\
We describe a \emph{simple star-shaped set} $S$ as a tuple $\langle\Delta_S,\{C_1,\dots,C_m\}\rangle$. $\Delta_S \subseteq \Delta$ is a set of domains on which the cuboids $\{C_1,\dots,C_m\}$ (and thus also $S$) are defined. We further require that the central region $P :=\textstyle\bigcap_{i = 1}^m C_i \neq \emptyset$. Then the simple star-shaped set $S$ is defined as 
$$S := \bigcup_{i=1}^m C_i$$
\end{definition}

In practice, it is often not possible to define clear-cut boundaries for concepts and properties. It is, for example, very hard to define a generally accepted crisp boundary for the property ``red''. We therefore use a fuzzified version of simple star-shaped sets for representing concepts, which allows us to define imprecise concept boundaries. This usage of fuzzy sets for representing concepts has already a long history (cf. \cite{Bvelohlavek2011,Douven2011,Osherson1982,Ruspini1991,Zadeh1982}).
We use a simple star-shaped set $S$ as a concept's ``core'' and define the membership of any point $x \in CS$ to this concept as $\max_{y \in S}Sim(x,y)$:
\begin{definition}
\label{def:FSSSS}
(Fuzzy simple star-shaped set)\\
A \emph{fuzzy simple star-shaped set} $\widetilde{S}$ is described by a quadruple $\langle S,\mu_0,c,W\rangle$ where
$S = \langle\Delta_S,\{C_1,\dots,C_m\}\rangle$ is a non-empty simple star-shaped set. The parameter $\mu_0 \in (0,1]$ controls the highest possible membership to $\widetilde{S}$ and is usually set to 1. The sensitivity parameter $c > 0$ controls the rate of the exponential decay in the similarity function. Finally, $W = \langle W_{\Delta_S},\{W_{\delta}\}_{\delta \in \Delta_S}\rangle$ contains positive weights for all domains in $\Delta_S$ and all dimensions within these domains, reflecting their respective importance. We require that $\textstyle\sum_{\delta \in \Delta_S} w_{\delta} = |\Delta_S|$ and that $\forall {\delta \in \Delta_S}:\textstyle\sum_{d \in \delta} w_{d} = 1$.\\
The membership function of $\widetilde{S}$ is then defined as follows:
$$\mu_{\widetilde{S}}(x) = \mu_0 \cdot \max_{y \in S}(e^{-c \cdot d_C^{\Delta_S}(x,y,W)})$$
\end{definition}

The sensitivity parameter $c$ controls the overall degree of fuzziness of $\widetilde{S}$ by determining how fast the membership drops to zero. The weights $W$ represent not only the relative importance of the respective domain or dimension for the represented concept, but they also influence the relative fuzziness with respect to this domain or dimension.
Note that if $|\Delta_S| = 1$, then $\widetilde{S}$ represents a property, and if $|\Delta_S| > 1$, then $\widetilde{S}$ represents a concept.

The right part of Figure \ref{fig:FSSSS} shows a fuzzy simple star-shaped set $\widetilde{S}$. In this illustration, the $x$ and $y$ axes are assumed to belong to different domains, and are combined with the Manhattan metric using equal weights.

\begin{proposition}
Any fuzzy simple star-shaped set $\widetilde{S} = \langle S,\mu_0,c,W\rangle$ is star-shaped with respect to $P = \textstyle\bigcap_{i=1}^{m} C_i$ under $d_C^{\Delta_S}$.
\end{proposition}
\begin{proof}
See appendix (\url{http://lucas-bechberger.de/appendix-ki-2017/}).
\end{proof}

\section{Operations on Concepts}
\label{Operations}

In this section, we define some operations on concepts (i.e., fuzzy simple star-shaped sets). The set of all concepts is closed under each of these operations.

\subsection{Intersection}
\label{Operations:Intersection}

If we intersect two simple star-shaped sets $S_1, S_2$, we simply need to intersect their cuboids. As an intersection of two cuboids is again a cuboid, the result of intersecting two simple star-shaped sets can be described as a union of cuboids. It is simple star-shaped if these resulting cuboids have a nonempty intersection. This is only the case if the central regions $P_1$ and $P_2$ of $S_1$ and $S_2$ intersect.\footnote{Note that if $D_{S_1} \cap D_{S_2} = \emptyset$, then $P_1 \cap P_2 \neq \emptyset$.}

However, we would like our intersection to result in a simple star-shaped set even if $P_1 \cap P_2 = \emptyset$. Thus, when intersecting two star-shaped sets, we might need to apply some repair mechanism in order to restore star-shapedness.\\

We propose to extend the cuboids $C_i$ of the intersection in such a way that they meet in some ``midpoint'' $p^* \in CS$ (e.g., the arithmetic mean of their centers). We create extended versions $C_i^{*}$ of all $C_i$ by defining their support points like this: 
$$\forall {d \in D}: p_{id}^{-*} := \min(p_{id}^-, p^*_d), \quad p_{id}^{+*} := \max(p_{id}^+, p^*_d)$$

The intersection of the resulting $C^{*}_i$ contains at least $p^*$, so it is not empty. This means that $S' = \langle\Delta_{S_1} \cup \Delta_{S_2}, \{C_1^{*},\dots,C_{m^*}^{*}\}\rangle$ is again a simple star-shaped set. We denote this modified intersection (consisting of the actual intersection and the application of the repair mechanism) as $S' = I(S_1,S_2).$\\

We define the intersection of two fuzzy simple star-shaped sets as $\widetilde{S}' = I(\widetilde{S}_1,\widetilde{S}_2) := \langle S',\mu'_0,c',W'\rangle$ with:
\begin{itemize}
	\item $S' := I(\widetilde{S}_1^{\alpha'},{\widetilde{S}}_2^{\alpha'})$ (where $\alpha' = \max\{\alpha \in [0,1]: {\widetilde{S}}_1^{\alpha} \cap {\widetilde{S}}_2^{\alpha} \neq \emptyset\}$)
	\item $\mu'_0 := \alpha'$
	\item $c' := \min(c^{(1)},c^{(2)})$
	\item $W'$ with weights defined as follows (where $s,t \in [0,1]$)\footnote{In some cases, the normalization constraint of the resulting domain weights might be violated. We can enforce this constraint by manually normalizing them.}:
	\begin{align*}
&\forall {\delta \in \Delta_{S_1} \cap \Delta_{S_2}}: \Big((w'_{\delta} := s \cdot w^{(1)}_{\delta} + (1-s) \cdot w^{(2)}_{\delta})\\
& \hspace{3.0cm}\land \forall {d \in \delta}: (w'_{d} := t \cdot w^{(1)}_{d} + (1-t) \cdot w^{(2)}_{d})\Big)\\
&\forall {\delta \in \Delta_{S_1} \setminus \Delta_{S_2}}: \Big((w'_{\delta} := w^{(1)}_{\delta}) \land \forall {d \in \delta}: (w'_{d} := w^{(1)}_{d})\Big)\\[-2pt]
&\forall {\delta \in \Delta_{S_2} \setminus \Delta_{S_1}}: \Big((w'_{\delta} := w^{(2)}_{\delta}) \land \forall {d \in \delta}: (w'_{d} := w^{(2)}_{d})\Big)
\end{align*}
\end{itemize}

When taking the combination of two somewhat imprecise concepts, the result should not be more precise than any of the original concepts. As the sensitivity parameter is inversely related to fuzziness, we take the minimum. If a weight is defined for both original sets, we take a convex combination, and if it is only defined for one of them, we simply copy it.

Note that for $\alpha' < \max(\mu_0^{(1)},\mu_0^{(2)})$, the $\alpha$-cuts $\widetilde{S}_1^{\alpha'}$ and $\widetilde{S}_2^{\alpha'}$ are still guaranteed to be star-shaped, but not necessarily simple star-shaped. In order to be still well-defined, the modified crisp intersection $I$  will in this case first compute their ``ordinary'' intersection, then approximate this intersection with cuboids (e.g., by using bounding boxes) and finally apply the repair mechanism.

\subsection{Union}
\label{Operations:Union}

As each simple star-shaped set is defined as a union of cuboids, the union of two such sets can also be expressed as a union of cuboids. However, the resulting set is not necessarily star-shaped -- only if the central regions of the original simple star-shaped sets intersect. So after each union, we might again need to perform a repair mechanism in order to restore star-shapedness. We propose to use the same repair mechanism that is also used for intersections. We denote the modified union as $S' = U(S_1, S_2)$.\\

We define the union of two fuzzy simple star-shaped sets as $\widetilde{S}' = U(\widetilde{S}_1, \widetilde{S}_2) := \langle S',\mu'_0,c',W'\rangle$ with:
\begin{itemize}
	\item $S' := U(S_1,S_2)$
	\item $\mu'_0 := \max(\mu_0^{(1)},\mu_0^{(2)})$
	\item $c'$ and $W'$ as described in Section \ref{Operations:Intersection}
\end{itemize}

\begin{proposition}
Let $\widetilde{S}_1 = \langle S_1, \mu_0^{(1)}, c^{(1)}, W^{(1)}\rangle$ and $\widetilde{S}_2 = \langle S_2, \mu_0^{(2)}, c^{(2)}, W^{(2)}\rangle$ be two fuzzy simple star-shaped sets. If we assume that $\Delta_{S_1} = \Delta_{S_2}$ and $W^{(1)} = W^{(2)}$, then $\widetilde{S}_1 \cup \widetilde{S}_2 \subseteq U(\widetilde{S}_1, \widetilde{S}_2) = \widetilde{S}'$.
\end{proposition}
\begin{proof}
See appendix (\url{http://lucas-bechberger.de/appendix-ki-2017/}).
\end{proof}

\subsection{Subspace Projection}
\label{Operations:Projection}

Projecting a cuboid onto a subspace results in a cuboid. As one can easily see, projecting a simple star-shaped set $S$ onto a subspace results in another simple star-shaped set. We denote the projection of $S$ onto domains $\Delta_{S'} \subseteq \Delta_S$ as $S' = P(S, \Delta_{S'})$.\\

We define the projection of a fuzzy simple star-shaped set $\widetilde{S}$ onto domains $\Delta_{S'} \subseteq \Delta_S$ as $\widetilde{S}' = P(\widetilde{S}, \Delta_{S'}) := \langle S', \mu'_0, c', W'\rangle$ with:
\begin{itemize}
	\item $S' := P(S,\Delta_{S'})$
	\item $\mu'_0 := \mu_0$
	\item $c' := c$
	\item $W' := \langle\{|\Delta_S'| \cdot \frac{w_{\delta}}{\sum_{\delta' \in \Delta_{S'}} w_{\delta'}}\}_{\delta \in \Delta_{S'}},\{W_{\delta}\}_{\delta \in \Delta_{S'}}\rangle$
\end{itemize}

Note that we only apply minimal changes to the parameters: $\mu_0$ and $c$ stay the same, only the domain weights are updated in order to not violate their normalization constraint.\\

Projecting a set onto two complementary subspaces and then intersecting these projections again in the original space yields a superset of the original set. This is intuitively clear for simple star-shaped sets and can also be shown for fuzzy simple star-shaped sets under one additional constraint:
\begin{proposition} 
Let $\widetilde{S} = \langle S, \mu_0, c, W\rangle$ be a fuzzy simple star-shaped set. Let $\widetilde{S}_1 = P(\widetilde{S}, \Delta_1)$ and $\widetilde{S}_2 = P(\widetilde{S}, \Delta_2)$ with $\Delta_1 \cup \Delta_2 = \Delta_S$ and $\Delta_1 \cap \Delta_2 = \emptyset$. Let $\widetilde{S}' = I(\widetilde{S}_1, \widetilde{S}_2)$ as described in Section \ref{Operations:Intersection}. If $\sum_{\delta \in \Delta_1} w_{\delta} = |\Delta_1|$ and $\sum_{\delta \in \Delta_2} w_{\delta} = |\Delta_2|$, then $\widetilde{S} \subseteq \widetilde{S}'$.
\end{proposition}
\begin{proof}
See appendix (\url{http://lucas-bechberger.de/appendix-ki-2017/}).
\end{proof}

\section{Supported Applications}
\label{Applications}

\subsection{Machine Learning Process: Clustering}
\label{Operations:Clustering}

The operations described in Section \ref{Operations} can be used by a clustering algorithm in the following way:

The clustering algorithm can create and delete fuzzy simple star-shaped sets. It can move and resize an existing cluster as well as adjust its form by modifying the support points of the cuboids that define its core. One must however ensure that such modifications preserve the non-emptiness of the cuboids' intersection. Moreover, a cluster's form can be changed by modifying the parameters $c$ and $W$: By changing $c$, one can control the overall degree of fuzziness, and by changing $W$, one can control how this fuzziness is distributed among the different domains and dimensions.
Two neighboring clusters $\widetilde{S}_1, \widetilde{S}_2$ can be merged into a single cluster by unifying them. A single cluster can be split up into two parts by replacing it with two smaller clusters.

So clusters can be created, deleted, modified, merged, and split -- which is sufficient for defining a clustering algorithm.

\subsection{Reasoning Process: Concept Combination}
\label{Operations:ConceptCombination}

The operations defined in Section \ref{Operations} can also be used for combining concepts.

The modified intersection $I(\widetilde{S}_1,\widetilde{S}_2)$ roughly corresponds to a logical ``AND'': Intersecting ``green'' with ``blue'' results in the set of all colors that are both green and blue to at least some degree. The modified union $U(\widetilde{S}_1, \widetilde{S}_2)$ can be used to construct higher-level categories: For instance, the concept of ``fruit'' can be obtained by the unification of ``apple'', ``banana'', ``pear'', ``pineapple'', etc.

G\"{a}rdenfors \cite{Gardenfors2000} argues that adjective-noun combinations like ``green apple'' or ``purple banana'' can be expressed by combining properties with concepts. This is supported by our operations of intersection and subspace projection:

In combinations like ``green apple'', property and concept are compatible. We expect that their cores intersect and that the $\mu_0$ parameter of their intersection is therefore relatively large. In this case, ``green'' should narrow down the color information associated with the ``apple'' concept. This can be achieved by simply computing their intersection.

In combinations like ``purple banana'', property and concept are incompatible. We expect that their cores do not intersect and that the $\mu_0$ parameter of their intersection is relatively small. In this case, ``purple'' should replace the color information associated with the ``banana'' concept. This can be achieved by first removing the color domain from the ``banana'' concept (through a subspace projection) and by then intersecting this intermediate result with ``purple''.

As one can see from this short discussion, our formalized framework is also capable of supporting reasoning processes.

\section{Related Work}
\label{RelatedWork}

This work is of course not the first attempt to devise an implementable formalization of the conceptual spaces framework.

An early and very thorough formalization was done by Aisbett \& Gibbon \cite{Aisbett2001}. Like we, they consider concepts to be regions in the overall conceptual space. However, they stick with G\"{a}rdenfors' assumption of convexity and do not define concepts in a parametric way. Their formalization targets the interplay of symbols and geometric representations, but it is too abstract to be implementable. 

Rickard et al. \cite{Rickard2006,Rickard2007} provide a formalization based on fuzziness. They represent concepts as co-occurence matrices of properties. By using some mathematical transformations, they interpret these matrices as fuzzy sets on the universe of ordered property pairs. Their representation of correlations is not geometrical: They first discretize the domains (by defining properties) and then compute the co-occurences between these properties. Depending on the discretization, this might lead to a relatively coarse-grained notion of correlation. Moreover, as properties and concepts are represented in different ways, one has to use different learning and reasoning mechanisms. Their formalization is also not easy to work with due to the complex mathematical transformations involved.

Adams \& Raubal \cite{Adams2009} represent concepts by one convex polytope per domain. This allows for efficient computations while being potentially more expressive than our cuboid-based representation. The Manhattan metric is used to combine different domains. However, correlations between different domains are not taken into account as each convex polytope is only defined on a single domain. Adams \& Raubal also define operations on concepts, namely intersection, similarity computation, and concept combination. This makes their formalization quite similar in spirit to ours. One could generalize their approach by using polytopes that are defined on the overall space and that are convex under the Euclidean and star-shaped under the Manhattan metric. However, we have found that this requires additional constraints in order to ensure starshapedness. The number of these constraints grows exponentially with the number of dimensions. Each modification of a concept's description would then involve a large constraint satisfaction problem, rendering this representation unsuitable for learning processes. Our cuboid-based approach is more coarse-grained, but it only involves a single constraint, namely that the intersection of the cuboids is not empty. 

Lewis \& Lawry \cite{Lewis2016} have recently formalized conceptual spaces using random set theory. They define properties as random sets within single domains and concepts as random sets in a boolean space whose dimensions indicate the presence or absence of properties. Their approach is similar to ours in using a distance-based membership function to a set of prototypical points. However, their work focuses on modeling concept combinations and does not explicitly consider correlations between domains. 

Many practical applications of conceptual spaces (e.g., \cite{Chella2003,Derrac2015,Dietze2008,Raubal2004}) use only partial ad-hoc implementations of the conceptual spaces framework which usually ignore some important aspects of the framework (e.g., the domain structure).

Finally, we can relate our work to statistical relational learning (SRL): Our geometric representation of concepts is a complex data structure (in SRL one typically uses logics for this) that is augmented with soft computing in the form of fuzziness (similar to the usage of probability theory in SRL).
\section{Conclusion and Future Work}
\label{Conclusion}

In this paper, we proposed a new formalization of the conceptual spaces framework. We aimed to geometrically represent correlations between domains, which led us to consider the more general notion of star-shapedness instead of G\"{a}rden\-fors' favored constraint of convexity. We defined concepts as fuzzy sets based on intersecting cuboids and a similarity-based membership function. Moreover, we provided different computationally efficient operations and illustrated that these operations can support both learning and reasoning processes.\\

This work is mainly seen as a theoretical foundation for an actual implementation of the conceptual spaces theory. In future work, we will enrich this formalization with additional operations. Moreover, we will devise a clustering algorithm that will work with the proposed concept representation. Both the mathematical framework presented in this paper and the clustering algorithm will be implemented and tested in practice which will provide valuable feedback.

\bibliographystyle{plain}
\bibliography{/home/lbechberger/Documents/Papers/jabref.bib}

\end{document}